\PassOptionsToPackage{numbers, compress}{natbib}
\documentclass{article}
\usepackage[preprint]{neurips_2026}

\usepackage[utf8]{inputenc} 
\usepackage[T1]{fontenc}    
\usepackage{hyperref}       
\usepackage{url}            
\usepackage{booktabs}       
\usepackage{amsfonts}       
\usepackage{nicefrac}       
\usepackage{microtype}      
\usepackage[table,dvipsnames]{xcolor}         
\usepackage{graphicx}
\usepackage{amsmath}
\usepackage{multirow}
\usepackage{algorithm}
\usepackage{algorithmic}
\usepackage{makecell}

\usepackage{amsthm}

\theoremstyle{plain}
\newtheorem{theorem}{Theorem}[section]
\newtheorem{proposition}{Proposition}

\DeclareMathOperator*{\argmax}{arg\,max}
\DeclareMathOperator*{\argmin}{arg\,min}

\title{When Losses Align: Gradient-Based Composite Loss Weighting for Efficient Pretraining}

\author{Ivan Karpukhin \\
Sber AI Lab\\
\texttt{iakarpukhin@sberbank.ru} \\
\And
Andrey Savchenko \\
Sber AI Lab \\
HSE University \\
ISP RAS Research Center for Trusted Artificial Intelligence\\
\texttt{avsavchenko@hse.ru}
}

\begin{document}

\maketitle
\vspace{-0.2in}
\begin{abstract} 

Modern deep models are often pretrained on large-scale data with missing labels using composite objectives, where the relative weights of multiple loss terms act as hyperparameters. Tuning these weights with random search or Bayesian optimization is computationally expensive, as it requires many independent training runs. To address this, we propose a gradient-based bilevel method that learns pretraining loss weights online by aligning the composite pretraining gradient with a downstream objective. By exploiting the structure of the loss, the method avoids the multiple backward passes typically required by truncated backpropagation through the full model, reducing the overhead of hyperparameter tuning to approximately 30\% above a single training run. We evaluate the approach on event-sequence modeling and self-supervised computer vision, where it matches or improves upon carefully tuned baselines while substantially reducing the cost of hyperparameter tuning compared to random or Bayesian search. The source code is released on GitHub: \url{https://github.com/ivan-chai/aligned-hpo}.

\end{abstract}

\section{Introduction}

Modern deep learning systems are commonly pretrained on large-scale unlabeled data and then adapted to downstream tasks~\cite{devlin2019bert,grill2020byol,babaev2022coles}. This two-stage paradigm is attractive because it encourages the model to learn reusable representations that transfer well, particularly when labeled data are limited. In many pretraining pipelines, however, the objective is not a single loss but a weighted combination of several loss terms~\cite{padhi2021tabgpt,estepa2023all4one,karpukhin2025ht}. The choice of these weights can materially affect representation quality and downstream performance, yet they are often selected manually or via black-box hyperparameter search, thus creating a practical difficulty.

Standard tuning strategies such as grid search, random search, or Bayesian optimization typically require many complete training runs, which becomes increasingly costly as the number of loss terms or training budget grows~\cite{snoek2012bayesian}. At the same time, naive default choices, such as uniform weighting, may fail to reflect the differing contributions of individual losses~\cite{yu2020gradientsurgery}. A more efficient way to tune loss weights is therefore needed, especially in large-scale pretraining settings where repeated full training runs are expensive.


In this work, we study the problem of tuning linear loss weights in composite pretraining objectives through bilevel optimization~\cite{franceschi2018bilevel}. Our goal is to optimize the loss weights such that the resulting pretrained representations improve downstream performance. We derive an SGD-based update rule that aligns the composite pretraining gradient with the downstream gradient. To reduce computational cost, we compute the alignment signal in the shared representation space, avoiding separate full backward passes for each loss term. The resulting method, which we call Gradient-aligned Pretraining (GraP), provides an efficient alternative to search-based hyperparameter optimization for composite objectives. In experiments on event-sequence modeling and self-supervised vision, GraP matches or improves carefully tuned baselines while requiring only a single training run with moderate overhead.


The main contributions of this paper are as follows:
\begin{enumerate}
\item We formulate pretraining loss-weight selection as a bilevel optimization problem targeting downstream performance.

\item We analyze normalization strategies and argue that the constraint should be imposed on the norm of the update vector, rather than on the weights themselves.

\item 
We propose an efficient gradient-based algorithm that updates loss weights with $O(1)$ complexity in the number of forward–backward steps with respect to the number of hyperparameters, by leveraging gradient alignment in the shared representation space.
\item 
We show empirically that the method achieves performance comparable to or better than standard hyperparameter optimization (e.g., Bayesian search) while being nearly $50\times$ more efficient, enabling hyperparameter tuning in settings where traditional methods are computationally infeasible. The proposed method also outperforms popular multi-task learning techniques on average.

\item 
We demonstrate that the proposed approach can identify redundant loss terms in composite objectives and remains effective across both sequential and vision-based pretraining tasks.
\end{enumerate}

\section{Problem formulation}
\label{sec:definition}
\begin{figure}[t]
    \centering
    \vspace{-0.1in}
    \includegraphics{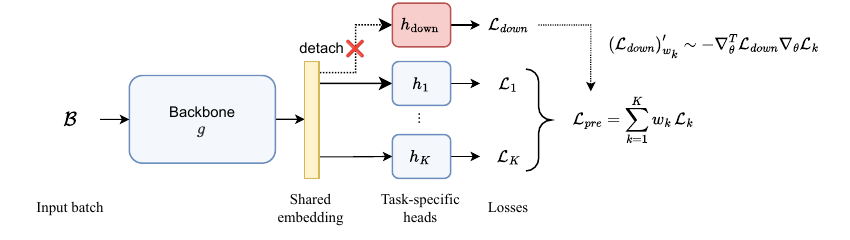}
    \vspace{-0.15in}
    \caption{Overview of the proposed GraP framework. A shared backbone produces embeddings used by multiple pretraining heads and a detached downstream head. Loss weights are updated by aligning the composite pretraining gradient with the downstream gradient in the shared representation space.}
    \label{fig:overview}
\end{figure}

This paper studies the problem of tuning weights in composite pretraining objectives to improve downstream task performance. Let $\mathcal{D} = \{x_i\}_{i=1}^N$ be a pretraining dataset, where a subset $\mathcal{D}_s = \{(x_i, y_i)\}$ contains downstream labels, while the remaining examples are unlabeled. Let $\mathcal{B}$ denote a minibatch sampled from $\mathcal{D}$, and let $\mathcal{B}_s \subseteq \mathcal{B}$ denote its labeled subset. We consider a model consisting of a shared backbone \(z = g(\mathcal{B}, \theta)\), which produces embeddings, and task-specific heads, as illustrated in Fig.~\ref{fig:overview}. Each pretraining objective \(\mathcal{L}_k,\; k = 1, \dots, K,\) is applied through a corresponding head \(h_k(z, \theta_k)\), while the downstream task uses a separate head \(h_\mathrm{down}(z, \theta_\mathrm{down})\) with loss \(\mathcal{L}_\mathrm{down}\). For clarity, we omit the dependence on model parameters in some expressions below.

The model is trained using a composite pretraining objective defined as
\begin{equation}
\mathcal{L}_{\text{pre}}(\mathcal{B}, \mathbf{w}) = \sum_{k=1}^K w_k \, \mathcal{L}_k\big(h_k(g(\mathcal{B}))\big),
\end{equation}
where $\{\mathcal{L}_k\}_{k=1}^K$ are task-specific loss terms and $\mathbf{w} = (w_1, \dots, w_K)$ are their weights. In this paper, we restrict the weights to be non-negative, as negative values correspond to gradient ascent on individual loss terms, which can lead to unstable updates and degradation of learned representations. This constraint is standard in multi-objective and multi-task optimization
~\cite{sener2018mtl}.

Our goal is to choose $\mathbf{w}$ such that the learned representation $z = g(\mathcal{B})$ performs well on the downstream task, as measured by the expected downstream loss
\begin{equation}
\mathcal{L}_{\text{down}}(\mathcal{B}_s) = \mathcal{L}_\mathrm{down}\big(h_\mathrm{down}(g(X)), Y\big),
\end{equation}
where $(X, Y)$ are inputs and labels from the minibatch $\mathcal{B}_s$.

To directly link pretraining with downstream performance, we cast this problem as a bilevel optimization task (we omit expectation over the dataset for simplicity):
\begin{equation}
\mathbf{w} = \argmin\limits_{\mathbf{w}, w_k > 0} \; \mathcal{L}_{\text{down}}(\theta^*(\mathbf{w}), \theta^*_\mathrm{down}(\mathbf{w})),
\label{eq:outer}
\end{equation}
subject to
\begin{equation}
\theta^*(\mathbf{w}), \theta^*_\mathrm{down}(\mathbf{w}) = \argmin\limits_{\theta,\theta_\mathrm{down}} \min\limits_{\theta_1, \dots, \theta_K} \; \left[\sum_{k=1}^K w_k \, \mathcal{L}_k(\theta, \theta_k) + \mathcal{L}_\mathrm{down}(\theta, \theta_\mathrm{down}) \right].
\label{eq:inner}
\end{equation}


This formulation explicitly captures the dependence of downstream performance on the choice of loss weights during pretraining.

\section{Method}
In this section, we first introduce the core idea underlying the proposed method and then present two extensions that improve representation quality and substantially increase computational efficiency.
\subsection{Stochastic gradient descent for pretraining loss weighting}
\label{sec:sgd-update}
We now derive an SGD-based update rule for the loss weights $\mathbf{w}$ under the bilevel formulation introduced in Sec.~\ref{sec:definition}. We consider a model consisting of a shared backbone $g$, pretraining heads $h_k$, and a downstream head $h_{\mathrm{down}}$. The pretraining objective is a weighted combination of $K$ loss terms, where the weights $\mathbf{w}$ are treated as hyperparameters. We assume that the inner optimization in Eq.~\ref{eq:inner} is performed using stochastic gradient descent (SGD). Below, we analyze a single SGD step, corresponding to the simplest form of truncated backpropagation~\cite{shaban2019truncated}. The extension to longer rollouts is discussed in Appendix~\ref{app:rollout}.

Given a minibatch $\mathcal{B}$ and its labeled subset $\mathcal{B}_s = (X, Y)$, we compute the pretraining gradients
\begin{equation}
g_k = \nabla_\theta \mathcal{L}_k\big(h_k(g(\mathcal{B}, \theta))\big), 
\quad k = 1, \dots, K,
\end{equation}
and form the composite pretraining gradient
\begin{equation}
g_{\mathrm{pre}}(\mathbf{w}) = \sum_{k=1}^K w_k g_k,
\end{equation}
which corresponds to the gradient of the weighted pretraining objective $\mathcal{L}_{\mathrm{pre}}$. The model parameters are then updated using SGD with learning rate $\eta$:
\begin{equation}
\theta' = \theta - \eta \, g_{\mathrm{pre}}(\mathbf{w}).
\end{equation}

Since $\theta'$ depends on $\mathbf{w}$, we can compute the gradient of the downstream loss with respect to the weights $\mathbf{w}$:
\begin{equation}
\label{eq:weightgrad}
\frac{\partial \mathcal{L}_{\mathrm{down}}}{\partial w_k}
=
g_{\mathrm{down}}^{\prime \top}
\frac{\partial \theta'}{\partial w_k}
=
-\eta \, g_{\mathrm{down}}^{\prime \top} g_k,
\end{equation}
\begin{equation}
g'_{\mathrm{down}} =
\nabla_{\theta'}
\mathcal{L}_{\mathrm{down}}
\big(h_{\mathrm{down}}(g(X, \theta'), Y)\big).
\end{equation}
Eq.~\ref{eq:weightgrad} shows that the hypergradient, i.e. the gradient of the downstream objective with respect to the loss weights, is proportional to the inner product between the downstream gradient and the $k$-th pretraining gradient. In practice, we assume that $\eta$ is sufficiently small and $g'_\mathrm{down} \approx g_\mathrm{down}$. Therefore, we compute the downstream gradient at the current parameters $\theta$ jointly with the pretraining gradients.

Let $G \in \mathbb{R}^{K \times d}$ denote the matrix whose $k$-th row is $g_k^\top$. Since updating $\mathbf{w}$ by gradient descent on $\mathcal{L}_{\mathrm{down}}$ is equivalent to maximizing the inner products $g_{\mathrm{down}}^\top g_k$, the corresponding local objective can be written as
\begin{equation}
\argmax_{\mathbf{w} \in \mathbb{R}^K, \, w_k \ge 0}
\;
\mathbf{w}^\top G g_{\mathrm{down}}.
\label{eq:opt-nonorm}
\end{equation}
We will use this objective to analyze the properties of the optimal loss weights.

A straightforward implementation of the weights update step computes each gradient $g_k = \nabla_\theta \mathcal{L}_k$ separately, resulting in $K$ full backward passes through the model. The key computational challenge is, therefore, to estimate these alignment scores efficiently without computing $K$ full gradients.

\subsection{Weight normalization and scale ambiguity}
\label{sec:normalization}

The absolute scale of the weights $\mathbf{w}$ directly affects the magnitude of the composite gradient and, consequently, the effective learning rate of the SGD update. Since scaling $\mathbf{w}$ by a constant rescales the update magnitude without changing the relative contributions of individual losses, the objective in Eq.~\ref{eq:opt-nonorm} exhibits scale ambiguity and therefore requires additional constraints on $\mathbf{w}$.

A common approach is to impose normalization constraints on the weights, such as enforcing $\sum_k w_k = 1$. This corresponds to the standard linear scalarization used in multi-task learning~\cite{sener2018mtl}, where weights control the relative contribution of different loss terms while constraining the overall magnitude of the weighted combination. Another option is to constrain the weights to have unit norm, e.g., $\|\mathbf{w}\| = 1$. However, such normalizations may lead to suboptimal composite-gradient directions, as illustrated in Fig.~\ref{fig:normalization}.

\begin{figure}[t]
    \centering
    \includegraphics{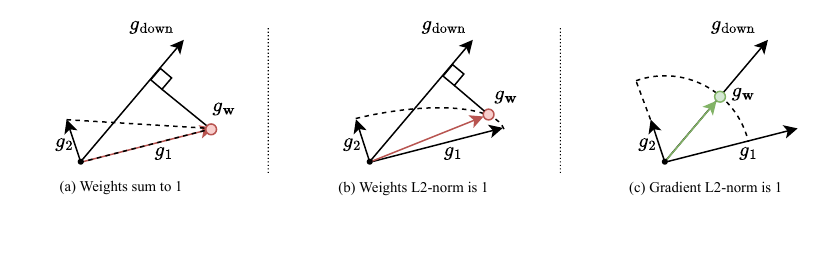}
    \vspace{-0.6in}
    \caption{Comparison of normalization strategies. The objective is to maximize the alignment between the composite gradient $g_\mathbf{w} = w_1 g_1 + w_2 g_2$ and the downstream gradient $g_{\mathrm{down}}$. Colored circles indicate the optimal solutions under different normalization schemes. Weight-sum and weight-norm normalization can distort the optimal composite-gradient direction, while composite-gradient normalization preserves the optimal alignment direction.}
    \label{fig:normalization}
\end{figure}

To address this issue, we constrain the $\ell_2$-norm of the composite gradient:
\begin{equation}
\left\| \sum_{k=1}^K w_k g_k \right\| = \left\| \mathbf{w}^\top G \right\| = 1.
\end{equation}
As shown in Appendix~\ref{app:normalization}, this scheme recovers the optimal alignment direction under the local objective from Eq.~\ref{eq:opt-nonorm}. The corresponding weights optimization objective has the following form:
\begin{equation}
\argmax_{\mathbf{w} \in \mathbb{R}^K, \, w_k \ge 0}
\; \frac{\mathbf{w}^\top G g_\mathrm{down}}{\left\| \mathbf{w}^\top G \right\|}
.
\label{eq:opt-norm}
\end{equation}


\subsection{Connection to $\ell_2$ Distance and Upper-Bound Approximation}
\label{sec:shared-space}

We first reinterpret the optimization objective from Eq.~\ref{eq:opt-norm} as a distance minimization problem:
\begin{equation}
\left\|
\frac{\mathbf{w}^\top G}{\left\| \mathbf{w}^\top G \right\|}
-
g_\mathrm{down}
\right\|^2
=
1 + \|g_\mathrm{down}\|^2
-
2
\frac{\mathbf{w}^\top G g_\mathrm{down}}
{\left\| \mathbf{w}^\top G \right\|}.
\end{equation}
Since \(g_\mathrm{down}\) does not depend on \(\mathbf{w}\), maximizing Eq.~\ref{eq:opt-norm} is equivalent to minimizing the distance between the normalized composite gradient and the downstream gradient.

Computing this objective directly requires access to the full parameter-space gradients \(g_k\) for all \(K\) loss terms. Following prior multi-task learning methods that compare gradients in a shared representation space~\cite{sener2018mtl,chen2018gradnorm,senushkin2023independent}, we instead compute gradients with respect to the backbone embedding \(z=g(\mathcal{B})\):
\begin{equation}
\tilde g_k =
\nabla_z \mathcal{L}_k(h_k(z)),
\qquad
\tilde g_\mathrm{down} =
\nabla_z \mathcal{L}_{\mathrm{down}}(h_\mathrm{down}(z)).
\end{equation}
Let \(\tilde G \in \mathbb{R}^{K \times d}\) denote the matrix whose \(k\)-th row is \(\tilde g_k^\top\). This leads to the following embedding-space objective, used in the proposed GraP method:
\begin{equation}
\argmax_{\mathbf{w} \in \mathbb{R}^K, \, w_k \ge 0}
\frac{\mathbf{w}^\top \tilde G \tilde g_\mathrm{down}}
{\left\| \mathbf{w}^\top \tilde G \right\|}.
\label{eq:grap-shared}
\end{equation}

The transition from parameter-space gradients to embedding-space gradients, as well as the replacement of the normalization term from Sec.~\ref{sec:normalization}, introduces additional assumptions and approximations. A detailed theoretical analysis, including upper-bound interpretation of the resulting objective, is provided in Appendix~\ref{app:embedding-space}.

Overall, the proposed approximation substantially reduces computational cost. Instead of computing \(K\) full backward passes through the backbone, we compute \(K\) gradients only from the losses to the shared embedding. Since the losses are applied through separate fully connected heads on top of the same embedding, these head-level backward passes are substantially cheaper than full backward passes through the backbone.

\subsection{Final algorithm}
We now summarize the resulting procedure in Alg.~\ref{alg:ntp}. At each iteration, we perform a forward pass through the backbone and all heads, then compute gradients of all loss terms with respect to the shared embedding \(z\). These gradients are used to update the weights \(\mathbf{w}\) based on their alignment with the downstream gradient. We then form the composite embedding-space gradient \(\mathbf{w}^\top \tilde G\) using the weights (before the update) and backpropagate it through the backbone. Finally, all model parameters are updated via SGD. Importantly, the gradients \(\tilde g_k\) are reused both for the weight update and the backbone update, eliminating the need for an additional backward pass through the combined loss.

To ensure stable training, each head is updated using its own unweighted gradient, allowing all heads to continue learning even when some weights \(w_k\) become zero. In contrast, the downstream gradient \(\tilde g_{\mathrm{down}}\) is used only to update the weights \(\mathbf{w}\) and the downstream head parameters \(\theta_{\mathrm{down}}\), and is not propagated through the backbone, since the pretraining stage remains unsupervised.

Overall, the proposed method requires only a single forward-backward pass through the model to update both the model parameters and the loss weights.

\begin{algorithm}[h!]
\caption{GraP algorithm for Loss Weight Tuning}
\label{alg:ntp}
\begin{algorithmic}[1]
\STATE $w_k \leftarrow 1, k=\overline{1,K}$ \hfill \COMMENT{Weights initialization}
\FOR{each minibatch $\mathcal{B}$}
    \STATE $z \leftarrow g(\mathcal{B})$ \hfill \COMMENT{embedding}
    \STATE $o_k \leftarrow h_k(z), \;\; o_\mathrm{down} \leftarrow h_\mathrm{down}(z)$ \hfill \COMMENT{head outputs}
    \STATE $\tilde g_k \leftarrow \nabla_z \mathcal{L}_k(o_k), \;\; \tilde g_\mathrm{down} \leftarrow \nabla_z \mathcal{L}_{\mathrm{down}}(o_\mathrm{down})$ \hfill \COMMENT{gradients}
    \STATE $\bar{\mathbf{w}} \leftarrow \mathbf{w} \,/\, \left\|\sum_k w_k \tilde g_k\right\|$ \hfill \COMMENT{composite gradient normalization}
    \STATE $\mathbf{w} \leftarrow \mathbf{w} + \eta \, \nabla_{\mathbf{w}} \left( \bar{\mathbf{w}}^\top \tilde G \, \tilde g_\mathrm{down} \right)$ \hfill \COMMENT{SGD step for weights}
    \STATE $\theta \leftarrow \theta - \eta \, \nabla_\theta \langle z, \bar{\mathbf{w}}^\top \tilde G \rangle$ \hfill \COMMENT{propagate $\bar{\mathbf{w}}^\top \tilde G$ to the backbone and update}
    \STATE update each $h_k$ using $\nabla_{\theta_k}\mathcal{L}_k(o_k)$ and $h_\mathrm{down}$ using $\nabla_{\theta_\mathrm{down}}\mathcal{L}_\mathrm{down}(o_\mathrm{down})$ \hfill
\ENDFOR
\end{algorithmic}
\end{algorithm}

\section{Experiments}
We conduct experiments in two domains: event sequences and computer vision. In both settings, we compare against a simple equal-weights baseline as well as several multi-task learning methods, including GradNorm~\cite{chen2018gradnorm}, Dynamic Weight Averaging (DWA)~\cite{liu2019dwa} MGDA~\cite{sener2018mtl}, and PCGrad~\cite{yu2020gradientsurgery}. We do not include uncertainty-based weighting methods~\cite{kendall2017uncertainty}, as they rely on task-specific probabilistic likelihood formulations that are not straightforward to define consistently across the heterogeneous mix of classification and regression objectives in our event-sequence setting. We follow the standard training pipelines provided by the benchmarks, with additional details included in Appendix~\ref{app:training-details}.

\subsection{Event sequences}
\textbf{Experiment design.}
We consider five real-world datasets of event sequences: \textit{Churn}, \textit{AgePred}, \textit{AlfaBattle}, \textit{MIMIC-III}, and \textit{Taobao}. We follow the standard preprocessing, pretraining, and evaluation protocol for event sequence modeling introduced in prior work~\cite{karpukhin2025ht}. These datasets cover a range of downstream tasks, including churn prediction, age group classification, credit default prediction, mortality prediction, and user activity prediction, as outlined in Appendix~\ref{app:training-details}.

Each dataset consists of event sequences with heterogeneous features, where each event contains multiple fields (typically 3--18). We formulate next-token prediction (NTP) as a composite objective combining cross-entropy losses for categorical features and mean absolute error for continuous values. This naturally yields a large number of loss components, making the setting suitable for evaluating loss-weight tuning methods. We additionally include a contrastive loss~\cite{babaev2022coles}, resulting in the following pretraining objective:
\begin{equation}
\mathcal{L}^\mathrm{seq}_\mathrm{pre} =
\sum\limits_{i=1}^F
w_i \mathcal{L}_{\mathrm{NTP}_i}
+
w_c \mathcal{L}_\mathrm{Contrastive},
\end{equation}
where \(F\) denotes the number of data fields.
During pretraining, we also tune a fully connected downstream classification head using cross-entropy loss. This head is used for weight tuning. Final model quality is evaluated by training a downstream classifier on top of frozen embeddings, using gradient boosting models as in prior work~\cite{sakhno2025pytorch}.

\textbf{Results.}
Table~\ref{tab:results} reports results on five event sequence datasets. The supervised RNN underperforms compared to unsupervised pretraining with equal weights, highlighting the benefits of pretraining when labeled data is limited. Overall, the proposed method improves average performance across both RNN and Transformer backbones compared to the equal-weights baseline and the evaluated multi-task learning methods. Notably, it consistently achieves the best results on the Churn and MIMIC-III datasets. Among the baselines, GradNorm~\cite{chen2018gradnorm} provides the strongest performance. For the full trajectories of the weights during training, please refer to Appendix~\ref{app:trajectories}.

\begin{table*}[h!]
\centering
\resizebox{\linewidth}{!}{
\begin{tabular}{l|ccccc|c}
    Method & Churn & AgePred & Alfabattle & MIMIC-III & Taobao & AVG \\ 
    \hline
    
\rowcolor{gray!10}Supervised RNN & 79.43{\tiny $\pm$ 0.64} & 61.48{\tiny $\pm$ 0.29} & {78.87{\tiny $\pm$ 0.18}} & {91.56{\tiny $\pm$ 0.24}} & {69.57{\tiny $\pm$ 0.20}} & 76.18\\
\hline
\hline
\rowcolor{BlueGreen!10}RNN Equal Weights & \bf 83.79{\tiny $\pm$ 0.79} & 63.76{\tiny $\pm$ 0.41} & 78.57{\tiny $\pm$ 0.00} & 91.25{\tiny $\pm$ 0.05} & 70.20{\tiny $\pm$ 0.49} & 77.51\\
\rowcolor{BlueGreen!10}RNN GradNorm & 83.41{\tiny $\pm$ 0.41} & {\bf 64.80{\tiny $\pm$ 0.32}} & 79.75{\tiny $\pm$ 0.28} & 90.72{\tiny $\pm$ 0.11} & 70.12{\tiny $\pm$ 0.31} & 77.76\\
\rowcolor{BlueGreen!10}RNN DWA & 83.74{\tiny $\pm$ 0.69} & 64.58{\tiny $\pm$ 0.18} & 78.49{\tiny $\pm$ 0.03} & 91.19{\tiny $\pm$ 0.07} & 70.27{\tiny $\pm$ 0.26} & 77.66\\
\rowcolor{BlueGreen!10}RNN MGDA & 77.79{\tiny $\pm$ 0.43} & 63.47{\tiny $\pm$ 0.27} & 78.66{\tiny $\pm$ 0.00} & 88.84{\tiny $\pm$ 0.23} & {70.42{\tiny $\pm$ 0.15}} & 75.84\\
\rowcolor{BlueGreen!10}RNN PCGrad & 83.77{\tiny $\pm$ 0.67} & 64.10{\tiny $\pm$ 0.29} & 78.55{\tiny $\pm$ 0.09} & 91.20{\tiny $\pm$ 0.10} & 70.23{\tiny $\pm$ 0.28} & 77.57\\
\hline
\rowcolor{BlueGreen!10}RNN GraP (Our) & 83.54{\tiny $\pm$ 0.71} & 64.31{\tiny $\pm$ 0.08} & {\bf 79.82{\tiny $\pm$ 0.07}} & {\bf 91.44{\tiny $\pm$ 0.05}} & {\bf 70.56{\tiny $\pm$ 0.14}} & \bf 77.94\\
\hline
\hline
\rowcolor{Apricot!10}Transformer Equal Weights & \bf 83.50{\tiny $\pm$ 0.95} & 61.27{\tiny $\pm$ 0.60} & 75.01{\tiny $\pm$ 0.05} & 89.37{\tiny $\pm$ 0.05} & 67.83{\tiny $\pm$ 0.55} & 75.39\\
\rowcolor{Apricot!10}Transformer GradNorm & 82.92{\tiny $\pm$ 0.65} & 62.31{\tiny $\pm$ 0.36} & {\bf 78.56{\tiny $\pm$ 0.06}} & 87.12{\tiny $\pm$ 1.72} & {69.80{\tiny $\pm$ 0.26}} & 76.14\\
\rowcolor{Apricot!10}Transformer DWA & 83.42{\tiny $\pm$ 0.58} & 58.82{\tiny $\pm$ 0.25} & 75.85{\tiny $\pm$ 0.21} & 89.01{\tiny $\pm$ 0.20} & 67.87{\tiny $\pm$ 0.50} & 74.99\\
\rowcolor{Apricot!10}Transformer MGDA & 82.24{\tiny $\pm$ 0.39} & 60.38{\tiny $\pm$ 0.22} & 77.03{\tiny $\pm$ 0.03} & 83.58{\tiny $\pm$ 0.35} & 69.72{\tiny $\pm$ 0.33} & 74.59\\
\rowcolor{Apricot!10}Transformer PCGrad & 82.87{\tiny $\pm$ 0.73} & 58.57{\tiny $\pm$ 0.32} & 75.79{\tiny $\pm$ 0.05} & 88.95{\tiny $\pm$ 0.15} & 68.05{\tiny $\pm$ 0.44} & 74.85\\
\hline
\rowcolor{Apricot!10}Transformer GraP (Our) & 83.45{\tiny $\pm$ 0.60} & {\bf 62.62{\tiny $\pm$ 0.25}} & 78.13{\tiny $\pm$ 0.07} & {\bf 91.28{\tiny $\pm$ 0.09}} & {\bf 69.90{\tiny $\pm$ 0.35}} & \bf 77.08\\
\hline
\hline
\rowcolor{RedOrange!10}HT-Transformer Equal Weights & 83.96{\tiny $\pm$ 0.65} & 54.76{\tiny $\pm$ 1.15} & 76.65{\tiny $\pm$ 0.22} & 87.39{\tiny $\pm$ 0.59} & 68.62{\tiny $\pm$ 0.19} & 74.28\\
\rowcolor{RedOrange!10}HT-Transformer GradNorm & 84.45{\tiny $\pm$ 0.36} & 62.24{\tiny $\pm$ 0.45} & {\bf 79.55{\tiny $\pm$ 0.01}} & 84.48{\tiny $\pm$ 0.70} & 69.83{\tiny $\pm$ 0.00} & 76.11\\
\rowcolor{RedOrange!10}HT-Transformer DWA & 84.11{\tiny $\pm$ 0.48} & 59.06{\tiny $\pm$ 0.78} & 76.84{\tiny $\pm$ 0.22} & 81.32{\tiny $\pm$ 0.21} & 68.28{\tiny $\pm$ 0.58} & 73.92\\
\rowcolor{RedOrange!10}HT-Transformer MGDA & 84.56{\tiny $\pm$ 0.40} & 60.17{\tiny $\pm$ 0.13} & 78.78{\tiny $\pm$ 0.24} & 83.71{\tiny $\pm$ 0.58} & {\bf 70.07{\tiny $\pm$ 0.00}} & 75.46\\
\rowcolor{RedOrange!10}HT-Transformer PCGrad & 84.22{\tiny $\pm$ 0.39} & 58.93{\tiny $\pm$ 0.37} & 76.52{\tiny $\pm$ 0.07} & 82.65{\tiny $\pm$ 2.51} & 68.71{\tiny $\pm$ 0.48} & 74.21\\
\hline
\rowcolor{RedOrange!10}HT-Transformer GraP (Our) & {\bf 85.21{\tiny $\pm$ 0.36}} & {\bf 63.37{\tiny $\pm$ 0.34}} & 79.48{\tiny $\pm$ 0.03} & {\bf 91.29{\tiny $\pm$ 0.03}} & 69.68{\tiny $\pm$ 0.00} & \bf 77.81\\

\end{tabular}
}
\caption{Event sequences results. The mean and standard deviation across 4 seeds are reported. The best method for each dataset-backbone combination is shown in bold.}
\label{tab:results}
\end{table*}

\textbf{Comparison with Optuna.}
In our standard setting, we combine NTP and contrastive objectives. Recent work provides Optuna-tuned weights for the NTP objective only~\cite{karpukhin2025ht}. In Table~\ref{tab:optuna}, we compare GraP with weights obtained via Optuna’s Bayesian optimization~\cite{akiba2019optuna}. The results show that our method achieves performance on par with weights tuned using Bayesian search.

\begin{table*}[h!]
\centering
\resizebox{\linewidth}{!}{
\begin{tabular}{l|ccccc|c}
    Method & Churn & AgePred & Alfabattle & MIMIC-III & Taobao & AVG \\ 
    \hline
    
RNN Optuna & 82.09{\tiny $\pm$ 0.92} & \bf 63.85{\tiny $\pm$ 0.41} & 80.64{\tiny $\pm$ 0.13} & \bf 91.29{\tiny $\pm$ 0.20} & \bf 69.26{\tiny $\pm$ 0.19} & \bf 77.43\\
RNN GraP (Our) & \bf 82.21{\tiny $\pm$ 0.55} & 63.83{\tiny $\pm$ 0.63} & \bf 80.72{\tiny $\pm$ 0.03} & 90.86{\tiny $\pm$ 0.21} & 69.19{\tiny $\pm$ 0.51} & 77.36\\
\hline
Transformer Optuna & 83.27{\tiny $\pm$ 0.62} & 63.93{\tiny $\pm$ 0.58} & 78.44{\tiny $\pm$ 0.14} & \bf 90.79{\tiny $\pm$ 0.07} & 68.94{\tiny $\pm$ 0.28} & 77.07\\
Transformer GraP (Our) & \bf 83.83{\tiny $\pm$ 0.90} & {\bf \bf 64.18{\tiny $\pm$ 0.35}} & \bf 78.62{\tiny $\pm$ 0.06} & 90.68{\tiny $\pm$ 0.38} & \bf 69.70{\tiny $\pm$ 0.28} & \bf 77.40\\
\hline
HT-Transformer Optuna & 84.57{\tiny $\pm$ 0.21} & 62.21{\tiny $\pm$ 0.61} & {\bf 81.09{\tiny $\pm$ 0.06}} & {\bf 91.76{\tiny $\pm$ 0.07}} & 70.25{\tiny $\pm$ 0.04} & 77.97\\
HT-Transformer GraP (Our) & {\bf 85.07{\tiny $\pm$ 0.39}} & \bf 62.31{\tiny $\pm$ 0.76} & 80.93{\tiny $\pm$ 0.13} & 91.69{\tiny $\pm$ 0.09} & {\bf 70.78{\tiny $\pm$ 0.20}} & \bf 78.16\\

\end{tabular}
}
\caption{Event sequences results without contrastive loss. The mean and standard deviation across 4 seeds are reported. The best method for each dataset-backbone combination is shown in bold.}
\label{tab:optuna}
\end{table*}

\subsection{Computer vision}
\textbf{Experiment design.}
We evaluate our method in the computer vision domain using the \texttt{solo-learn}~\cite{da2022solo}, a standardized framework for self-supervised representation learning. We consider experiments on \textit{CIFAR-10}, \textit{CIFAR-100}~\cite{krizhevsky2009cifar}, and \textit{ImageNet-100}, a subset of ImageNet~\cite{krizhevsky2012imagenetcls}. We apply our approach to \textit{All4One}~\cite{estepa2023all4one}, a recent self-supervised method that combines multiple objectives inspired by Barlow Twins~\cite{zbontar2021barlow}, BYOL~\cite{grill2020byol}, and NNCLR~\cite{dwibedi2021nnclr}. Model quality is evaluated following the benchmark protocol via online training of a linear classifier on top of detached representations, reporting validation top-1 accuracy. 
In this setting, we observed that weight perturbations at early training stages can affect the final performance. To address this, we report a \textit{GraP Tuned} variant, where All4One is trained from scratch using median weights computed by our method.

\textbf{Results.}
Table~\ref{tab:results-ntp-hpo-cv} summarizes the results. Overall, all methods demonstrate similar performance across datasets, indicating that this setting is relatively insensitive to the exact choice of loss weights. The best-performing method varies across datasets: GradNorm achieves the highest average score, DWA performs best on ImageNet-100, and GraP-tuned weights achieve the best result on CIFAR-10. However, these differences are typically small compared to the standard deviation. Notably, our method achieves performance comparable to the default All4One parameter set without requiring manual tuning or an expensive hyperparameter search. 
Interestingly, our method assigns a near-zero weight to the Attn-NNCLR loss. Removing this loss (All4One W/O Attn-NNCLR) yields performance similar to the default model, suggesting that this objective has a limited impact in this setting. This behavior indicates that the proposed method can identify redundant or low-impact losses and, more generally, serve as an automatic loss-selection mechanism. 
Overall, these results indicate that in standard self-supervised vision settings, our approach provides a reliable alternative to hyperparameter tuning and multi-task learning methods.

\begin{table*}[h!]
\centering
\resizebox{.8\linewidth}{!}{
\begin{tabular}{l|ccc|c}
    Method & CIFAR-10 & CIFAR-100 & ImageNet-100 & AVG \\ 
    \hline

All4One (default) & 93.08{\tiny $\pm$ 0.12} & 71.48{\tiny $\pm$ 0.16} & {85.92{\tiny $\pm$ 0.03}} & 83.49\\
All4One W/O Attn-NNCLR & 93.11{\tiny $\pm$ 0.10} & 71.32{\tiny $\pm$ 0.47} & 85.65{\tiny $\pm$ 0.06} & 83.36\\

\hline

All4One Equal Weights & 92.70{\tiny $\pm$ 0.16} & 71.45{\tiny $\pm$ 0.14} & 86.08{\tiny $\pm$ 0.10} & 83.41\\
All4One GradNorm & 93.01{\tiny $\pm$ 0.04} & {\bf 71.81{\tiny $\pm$ 0.23}} & 85.81{\tiny $\pm$ 0.07} & \bf 83.54\\
All4One DWA & 92.72{\tiny $\pm$ 0.11} & 71.44{\tiny $\pm$ 0.34} & {\bf 86.31{\tiny $\pm$ 0.05}} & 83.49\\
All4One MGDA & 92.81{\tiny $\pm$ 0.11} & 69.98{\tiny $\pm$ 0.19} & 83.79{\tiny $\pm$ 0.03} & 82.19\\
\hline
All4One GraP & 92.93{\tiny $\pm$ 0.17} & 70.90{\tiny $\pm$ 0.06} & 84.18{\tiny $\pm$ 0.12} & 82.67\\
All4One GraP Tuned & {\bf 93.16{\tiny $\pm$ 0.09}} & 71.51{\tiny $\pm$ 0.35} & 85.81{\tiny $\pm$ 0.08} & 83.49\\

\end{tabular}
}
\caption{Image classification results. The mean and standard deviation across 4 seeds are reported. The best method for each dataset is shown in bold.}
\label{tab:results-ntp-hpo-cv}
\end{table*}

\subsection{Ablation studies}
\textbf{On potential label leakage.}
Our method updates the loss weights using labeled data, raising the question of whether this introduces label leakage and effectively makes backbone training partially supervised. We address this concern in several ways. First, the weights influence the backbone only indirectly through the composite gradient. The backbone and weight updates are computed simultaneously, while the weights evolve gradually, typically following smooth and monotonic trajectories (see Appendix~\ref{app:trajectories}). As a result, the weights cannot adapt to individual minibatches, limiting the transfer of label-specific information to the backbone. Second, we evaluate the effect of the amount of labeled data used for weight tuning on the Churn and MIMIC-III datasets. Table~\ref{tab:label_fraction} shows that the method is robust to the fraction of labeled samples, suggesting that labels primarily provide a weak guidance signal for weight selection rather than directly supervising representation learning. Finally, the supervised fine-tuning results in Table~\ref{tab:aux} demonstrate that supervision still provides additional information not captured during pretraining. In this sense, the proposed method is comparable to standard hyperparameter optimization in terms of supervision from labels.


\begin{table*}[h!]
\centering
\begin{tabular}{l|ccccc}
    Dataset & 5\% & 10\% & 20\% & 50\% & 100\% \\
    \hline
    Churn & 83.61{\tiny $\pm$ 0.48} & 83.32{\tiny $\pm$ 0.12} & 84.80{\tiny $\pm$ 0.15} & 84.00{\tiny $\pm$ 0.70} & 84.45{\tiny $\pm$ 0.55} \\
    MIMIC-III & 91.36{\tiny $\pm$ 0.02} & 91.74{\tiny $\pm$ 0.04} & 91.58{\tiny $\pm$ 0.14} & 91.58{\tiny $\pm$ 0.06} & 91.50{\tiny $\pm$ 0.05} \\
\end{tabular}
\caption{Classification results on the Churn dataset for various fractions of labeled data.}
\label{tab:label_fraction}
\end{table*}

\textbf{Weights normalization.}
In Sec.~\ref{sec:normalization}, we introduced a normalization strategy based on the norm of the weighted composite gradient. Table~\ref{tab:normalization} compares the proposed approach with standard weight-sum normalization~\cite{sener2018mtl} and a variant with unconstrained weights. The proposed normalization improves average performance across datasets, with the largest gains observed in settings with relatively low variance across runs.
\begin{table*}[h!]
\centering
\begin{tabular}{l|ccccc|c}
    Method & Churn & AgePred & Alfabattle & MIMIC-III & Taobao & AVG \\ 
    \hline
GraP No norm. & {\bf 84.26{\tiny $\pm$ 0.55}} & 63.86{\tiny $\pm$ 0.41} & 78.39{\tiny $\pm$ 0.05} & 91.27{\tiny $\pm$ 0.11} & 69.65{\tiny $\pm$ 0.14} & 77.48\\
GraP Sum norm. & 83.80{\tiny $\pm$ 0.58} & 63.43{\tiny $\pm$ 0.61} & 79.12{\tiny $\pm$ 0.12} & 91.25{\tiny $\pm$ 0.20} & {\bf 70.78{\tiny $\pm$ 0.32}} & 77.67\\
GraP & 83.54{\tiny $\pm$ 0.71} & {\bf 64.31{\tiny $\pm$ 0.08}} & {\bf 79.82{\tiny $\pm$ 0.07}} & {\bf 91.44{\tiny $\pm$ 0.05}} & 70.56{\tiny $\pm$ 0.14} & \bf 77.94\\

\end{tabular}
\caption{Comparison of weights normalization strategies.}
\label{tab:normalization}
\end{table*}

\textbf{Auxiliary training.}
In this section, we compare the standard pretraining + fine-tuning paradigm with auxiliary training, where the supervised loss is optimized jointly with the pretraining objectives. Table~\ref{tab:aux} shows that auxiliary training improves over the unsupervised equal-weights baseline, as expected. However, pretraining followed by supervised fine-tuning achieves better average performance, likely due to reduced overfitting on limited labeled data. GraP further improves over auxiliary training even in the unsupervised setting, while supervised fine-tuning yields the best overall results.
\begin{table*}[h!]
\centering
\begin{tabular}{l|ccccc|c}
Method & Churn & AgePred & Alfabattle & MIMIC-III & Taobao & AVG \\ 
\hline
Supervised & 79.43{\tiny $\pm$ 0.64} & 61.48{\tiny $\pm$ 0.29} & 78.87{\tiny $\pm$ 0.18} & 91.56{\tiny $\pm$ 0.24} & 69.57{\tiny $\pm$ 0.20} & 76.18\\
Auxiliary & {\bf 84.14{\tiny $\pm$ 0.46}} & 64.28{\tiny $\pm$ 0.24} & 78.81{\tiny $\pm$ 0.07} & 91.26{\tiny $\pm$ 0.15} & 70.49{\tiny $\pm$ 0.33} & 77.80\\
\hline
Equal Weights & 83.79{\tiny $\pm$ 0.79} & 63.76{\tiny $\pm$ 0.41} & 78.57{\tiny $\pm$ 0.00} & 91.25{\tiny $\pm$ 0.05} & 70.20{\tiny $\pm$ 0.49} & 77.51\\
Equal Weights Tuned & 83.44{\tiny $\pm$ 0.46} & 64.32{\tiny $\pm$ 0.70} & 79.60{\tiny $\pm$ 0.04} & {\bf 92.23{\tiny $\pm$ 0.18}} & 70.40{\tiny $\pm$ 0.24} & 78.00\\
\hline
GraP & 83.54{\tiny $\pm$ 0.71} & 64.31{\tiny $\pm$ 0.08} & 79.82{\tiny $\pm$ 0.07} & 91.44{\tiny $\pm$ 0.05} & {\bf 70.56{\tiny $\pm$ 0.14}} & 77.94\\
GraP Tuned & 83.09{\tiny $\pm$ 0.81} & {\bf 64.78{\tiny $\pm$ 0.55}} & {\bf 80.66{\tiny $\pm$ 0.19}} & 92.05{\tiny $\pm$ 0.06} & 70.40{\tiny $\pm$ 0.30} & \bf 78.20\\

\end{tabular}
\caption{Comparison with auxiliary training.}
\label{tab:aux}
\end{table*}

\subsection{Computational Cost and Memory Usage}

\textbf{Computation speed.}
Table~\ref{tab:speed} compares the computational cost of a single training run, GraP, and classical hyperparameter optimization methods for RNN-based models on event-sequence datasets. Across datasets, our method introduces a moderate overhead compared to standard training, with an average slowdown of $44\%$ in throughput and $34\%$ in epoch time. The exact overhead depends on the dataset and model configuration, but remains within a constant factor (Sec.~\ref{sec:shared-space}). In contrast, Bayesian optimization and TPE typically require dozens of independent training runs~\cite{snoek2012bayesian,bergstra2013sciencehpo}, resulting in orders-of-magnitude higher computational cost. Under a conservative estimate of 50 runs, the effective overhead ranges from $5000\%$ to $18000\%$. Thus, the proposed method achieves competitive performance while requiring only a single training run with moderate additional computation.

\begin{table}[h!]
\centering
\resizebox{\textwidth}{!}{
\begin{tabular}{l|l|ccccc|c}
    Metric & Method & Churn & AgePred & AlfaBattle & MIMIC-III & TaoBao & AVG \\ 
    \hline
    \multirow{3}{*}{RPS} & Baseline & 2062 & 110 & 152 & 68 & 94 \\
    & GraP & 1266 & 82 & 86 & 54 & 79 \\
    & \textit{Overhead} & \textit{63\%} & \textit{34\%} & \textit{79\%} & \textit{26\%} & \textit{19\%} & \textit{44\%} \\
    \hline
    \multirow{3}{*}{Epoch time (s)} & Baseline & 1.64 & 35.7 & 552.2 & 58.7 & 31.5 \\
    & GraP & 2.16 & 46.6 & 946.8 & 70.2 & 36.9 \\
    & \textit{Overhead} & \textit{32\%} & \textit{31\%}  & \textit{71\%} & \textit{20\%} & \textit{17\%} & \textit{34\%} \\
    \hline
    \multicolumn{2}{l|}{Expected Bayesian / TPE Overhead} & \textit{6000\%} & \textit{5000\%} & \textit{18000\%} & \textit{5000\%} & \textit{5000\%} & \textit{7800\%} \\
\end{tabular}
}
\caption{RNN training and hyperparameter tuning speed comparison.}
\label{tab:speed}
\end{table}

\textbf{Memory usage.}
Our approach does not require multiple model replicas or optimization states across runs. The additional memory overhead comes from storing per-loss gradients at the embedding level, which is typically small compared to the backbone size in computer vision models. However, for smaller sequential models, the relative overhead can be more noticeable. For example, memory usage increases by a factor of $2.8$ on AgePred and $3.7$ on AlfaBattle compared to standard training.

\section{Related Work}

\textbf{Classical hyperparameter optimization.} Classical HPO treats training as a black-box procedure and searches over configurations using grid search, random search, Bayesian optimization, or related model-based methods~\cite{snoek2012bayesian,bergstra2013sciencehpo}. While effective, these approaches typically require many independent training runs, which can make them expensive for large-scale pretraining with multiple loss terms.

\textbf{Differentiable hyperparameter optimization.}
Gradient-based hyperparameter optimization instead differentiates validation performance with respect to hyperparameters such as learning rates, initialization parameters, and regularization strengths~\cite{maclaurin2015gradient,fu2016drmad,franceschi2018bilevel}. Subsequent work reduced the computational cost using truncated backpropagation or implicit differentiation~\cite{luketina2016scalable,lorraine2020optimizing}. Closest to our setting, \citet{luketina2016scalable} optimize continuous regularization hyperparameters online with a small constant overhead. In contrast, we focus on loss-function hyperparameters and exploit the objective's structure to avoid the separate backward passes required for each loss term.

\textbf{Multi-task and multi-objective learning.}
Learning with multiple objectives is central to multi-task learning, where task gradients may conflict and simple scalarization can lead to suboptimal trade-offs~\cite{sener2018mtl,quinton2024jacobian}. Recent works study scalable loss balancing, uncertainty-based weighting, and auxiliary-task weighting~\cite{kendall2017uncertainty,li2024unprejudiced,gregoire2024sample,xiao2025ldc}. Unlike standard multi-task learning, our goal is not to jointly optimize all objectives, but to tune pretraining loss weights to improve downstream transfer performance.

\textbf{Auxiliary learning.}
Several works adapt auxiliary-task weights using gradient similarity or related heuristics~\cite{du2018auxadapt,lin2019adaptiveaux}. Closest to our work, \citet{lin2019adaptiveaux} use gradient alignment between auxiliary and main tasks to adapt task weights in reinforcement learning. Our approach is conceptually related, but focuses on the pretraining task and introduces an efficient optimization procedure based on shared-representation gradients and composite-gradient normalization.

\section{Limitations}

The proposed method has several limitations. First, it assumes that each pretraining loss is applied via a separate head. In settings where multiple losses share parameters within a common head, the gradients with respect to the shared representation can still be computed independently via separate backward passes. However, tuning the shared head itself may be more challenging, as updates driven by one loss can suppress or interfere with signals from others. Extending the method to better handle such interactions at the head level is an important direction for future work.

Second, when multiple losses are highly correlated, the method may prioritize a subset that is most aligned with the downstream objective. While this simplifies the effective objective, it may underutilize complementary signals, particularly in the presence of gradient noise. Extensions that account for correlation or uncertainty (e.g., probabilistic weighting) may further improve robustness.

Finally, proposed updates can affect the magnitude of the composite gradient, whereas many training pipelines rely on carefully tuned learning rate schedules. As a result, adjusting the gradient scale may be beneficial to maintain stable optimization.

\section{Conclusion}
In this paper, we studied the problem of tuning loss weights in composite pretraining objectives and formulated it as a bilevel optimization task targeting downstream performance. We proposed a gradient-based method that updates the weights by aligning the pretraining update with the downstream gradient, while reducing computational cost through shared-representation optimization and gradient reuse.

Empirically, the method improves average performance over baselines and carefully tuned hyperparameters across event-sequence modeling tasks while requiring only a single training run with moderate overhead. In self-supervised vision, where objectives are typically better balanced and involve fewer components, the approach remains competitive with standard tuning strategies. Overall, the results suggest that gradient-based loss-weight optimization is an efficient alternative to classical multi-task learning and search-based hyperparameter optimization methods, particularly in settings with many interacting loss terms or limited computational budgets.

\bibliographystyle{unsrtnat}
\bibliography{main}

\appendix

\newpage

\appendix
\section{Multi-step SGD Rollout}
\label{app:rollout}

In Sec.~\ref{sec:sgd-update}, we derived the loss-weight update using a single-step SGD approximation.
We now extend this analysis to a multi-step rollout and show that, to first order,
the required hypergradient depends on the accumulated pretraining gradients along the trajectory.

\paragraph{Setup.}
Recall the composite pretraining loss
\[
L_{\mathrm{pre}}(\theta, \mathbf{w}) = \sum_{k=1}^K \mathbf{w}_k L_k(\theta),
\]
and consider an SGD trajectory of length $n$:
\[
\theta_{t+1}
=
\theta_t
-
\eta
\sum_{k=1}^K \mathbf{w}_k \nabla_\theta L_k(\theta_t),
\quad t = 0, \dots, n.
\]
We are interested in the hypergradient of the downstream loss
$L_{\mathrm{down}}(\theta_{n+1})$ with respect to $\mathbf{w}_i$.

\begin{theorem}[First-order multi-step hypergradient]
\label{thm:multi_step}
Assume that each $L_k$ is twice continuously differentiable and that
$\|\nabla^2_\theta L_i(\theta)\| \le C$ for all $k$ along the trajectory.
For a fixed rollout length $n$ and sufficiently small step size $\eta$,
the hypergradient satisfies
\[
\frac{\partial L_{\mathrm{down}}(\theta_{n+1})}{\partial \mathbf{w}_i}
=
-\eta
\nabla_\theta L_{\mathrm{down}}(\theta_{n+1})^\top
\sum_{t=0}^{n}
\nabla_\theta L_i(\theta_t)
+
O(n^2 \eta^2).
\]
\end{theorem}

\begin{proof}
Let vector
\[
\Omega_i^t := \frac{\partial \theta_t}{\partial w_i}
\]
denote the sensitivity of the parameters to the weight $\mathbf{w}_i$.
Since $\theta_0$ does not depend on $\mathbf{w}_i$, we have $\Omega_i^0 = 0$.

Differentiating the SGD update yields
\[
\Omega_i^{t+1}
= \frac{\partial \theta_{t + 1}}{\partial \mathbf{w}_i} =
\Omega_i^t
-
\eta \nabla_\theta L_i(\theta_t)
-
\eta
\sum_{k=1}^K
\mathbf{w}_k
\nabla_\theta^2 L_k(\theta_t)\,\Omega_i^t.
\]

We analyze the magnitude of $\Omega_i^t$.
From the recursion, each step contributes a term of size $O(\eta)$,
so for fixed $n$,
\[
\|\Omega_i^t\| = O(t \eta) = O(n \eta).
\]

Substituting this into the Hessian term gives
\[
\left\|
\eta \sum_{k=1}^K \mathbf{w}_k \nabla_\theta^2 L_k(\theta_t)\,\Omega_i^t
\right\|
\le
\eta C \|\Omega_i^t\|
=
O(n \eta^2).
\]

Thus, the recursion becomes
\[
\Omega_i^{t+1}
=
\Omega_i^t
-
\eta \nabla_\theta L_i(\theta_t)
+
O(n \eta^2).
\]

Unrolling from $t=0$ to $n$ gives
\[
\Omega_i^{n+1}
=
-\eta
\sum_{t=0}^{n}
\nabla_\theta L_i(\theta_t)
+
O(n^2 \eta^2).
\]

Finally, applying the chain rule,
\[
\frac{\partial L_{\mathrm{down}}(\theta_{n+1})}{\partial \mathbf{w}_i}
=
\nabla_\theta L_{\mathrm{down}}(\theta_{n+1})^\top \Omega_i^{n+1},
\]
which yields the desired result.
\end{proof}

\paragraph{Discussion.}
Theorem~\ref{thm:multi_step} shows that, up to first order in $\eta$,
the hypergradient is determined by the alignment between the downstream gradient
at the end of the rollout and the \emph{accumulated} gradients of each pretraining
loss along the trajectory.

Importantly, all steps contribute at the same order $O(\eta)$, while the deviation
from this accumulation arises from Hessian-dependent transport terms, which are
of higher order $O(n^2 \eta^2)$. Thus, the one-step approximation used in
Sec.~4.1 corresponds to the special case $n=0$.

In practice, the accumulated gradients can be approximated using an exponential
moving average (EMA):
\[
m_i^{t+1} = \beta m_i^t + (1-\beta)\nabla_\theta L_i(\theta_t),
\]
and the alignment score becomes
\[
\nabla_\theta L_{\mathrm{down}}(\theta_t)^\top m_i^t.
\]

In preliminary experiments, we did not observe improvements from using exponential moving average (EMA) gradients. This is likely because the trajectory of the weight vector under SGD is already smooth, with the degree of smoothing effectively controlled by the learning rate. Moreover, maintaining EMA gradients introduces additional memory overhead and is not compatible with the proposed upper-bound approximation.

\section{Composite Gradient Normalization}
\label{app:normalization}

The objective in Eq.~\ref{eq:opt-nonorm} is scale-invariant with respect to the loss weights $\mathbf{w}$, leading to ambiguity in the magnitude of the composite gradient. To remove this ambiguity, we normalize the weighted pretraining gradient by its norm. This yields the following objective:
\begin{equation}
\argmax_{\mathbf{w} \in \mathbb{R}^K, \, w_k \ge 0}
\frac{\mathbf{w}^\top G g_\mathrm{down}}{\left\| \mathbf{w}^\top G \right\|}.
\end{equation}

Since scaling $g_\mathrm{down}$ does not affect the optimum, the objective is equivalent to maximizing the cosine similarity between the composite pretraining gradient and the downstream gradient:
\begin{equation}
\argmax_{\mathbf{w} \in \mathbb{R}^K, \, w_k \ge 0, \, \|\mathbf{w}\| = 1}
\cos\!\left(\mathbf{w}^\top G,\, g_\mathrm{down}\right).
\end{equation}

The following proposition shows that the optimum corresponds to directional alignment between the composite pretraining gradient and the downstream gradient.

\begin{proposition}
Assume that there exists $\mathbf{w}$ such that
\[
\mathbf{w}^\top G = \alpha g_{\mathrm{down}}
\]
for some $\alpha > 0$. Then the objective
\begin{equation}
\max_{\mathbf{w}}
\frac{\mathbf{w}^\top G g_{\mathrm{down}}}
{\|\mathbf{w}^\top G\|}
\end{equation}
is maximized when the composite pretraining gradient $\mathbf{w}^\top G$ is aligned with the downstream gradient $g_{\mathrm{down}}$.
\end{proposition}

\begin{proof}
By the Cauchy--Bunyakovsky--Schwarz inequality,
\begin{equation}
\frac{\mathbf{w}^\top G g_{\mathrm{down}}}
{\|\mathbf{w}^\top G\|}
\le
\|g_{\mathrm{down}}\|,
\end{equation}
with equality if and only if
\[
\mathbf{w}^\top G = \alpha g_{\mathrm{down}}
\]
for some $\alpha > 0$.
\end{proof}

Thus, in the ideal case, the normalized objective recovers a composite pretraining gradient whose direction coincides with the downstream gradient. This motivates the use of composite-gradient normalization when updating the loss weights $\mathbf{w}$.

\section{Embedding-Space Approximation}
\label{app:embedding-space}

In Sec.~\ref{sec:shared-space}, we replace parameter-space gradient comparisons with comparisons in the shared representation space. This section provides additional justification for this approximation and discusses the assumptions underlying the resulting objective.

\paragraph{Assumptions.}
We assume that the backbone mapping \(z = g(x, \theta)\) is differentiable with respect to \(\theta\), and that its Jacobian
\[
J_\theta = \frac{\partial z}{\partial \theta}
\]
has bounded operator norm in a neighborhood of the optimization trajectory. Under this assumption, parameter-space gradients can be expressed through embedding-space gradients via the chain rule, and the corresponding gradient mismatch in parameter space can be controlled by the mismatch in the shared representation space.

\paragraph{Distance interpretation.}
Recall the normalized parameter-space objective introduced in Sec.~\ref{sec:normalization}:
\begin{equation}
\argmax_{\mathbf{w} \in \mathbb{R}^K, \, w_k \ge 0}
\frac{\mathbf{w}^\top G g_\mathrm{down}}
{\left\| \mathbf{w}^\top G \right\|}.
\label{eq:param-obj}
\end{equation}

This objective can be interpreted as minimizing the distance between the normalized composite gradient and the downstream gradient:
\begin{equation}
\left\|
\frac{\mathbf{w}^\top G}
{\left\| \mathbf{w}^\top G \right\|}
-
g_\mathrm{down}
\right\|^2
=
1 + \|g_\mathrm{down}\|^2
-
2
\frac{\mathbf{w}^\top G g_\mathrm{down}}
{\left\| \mathbf{w}^\top G \right\|}.
\end{equation}
Since \(g_\mathrm{down}\) does not depend on \(\mathbf{w}\), maximizing Eq.~\ref{eq:param-obj} is equivalent to minimizing the corresponding distance.

\paragraph{Upper-bound approximation.}
Following prior multi-task learning methods~\cite{sener2018mtl,chen2018gradnorm,senushkin2023independent}, we compare gradients in the shared representation space rather than in the full parameter space. Let
\[
z = g(\mathcal{B}, \theta)
\]
denote the backbone embedding, and define the embedding-space gradients
\begin{equation}
\tilde g_k =
\nabla_z \mathcal{L}_k(h_k(z)),
\qquad
\tilde g_\mathrm{down} =
\nabla_z \mathcal{L}_{\mathrm{down}}(h_\mathrm{down}(z)).
\end{equation}

By the chain rule,
\begin{equation}
g_k = J_\theta^\top \tilde g_k,
\qquad
g_\mathrm{down} = J_\theta^\top \tilde g_\mathrm{down}.
\end{equation}

Using these relations, the parameter-space mismatch can be bounded by the corresponding mismatch in the embedding space.

\begin{theorem}
For any \(\mathbf{w}\),
\begin{equation}
\left\|
\mathbf{w}^\top G - g_\mathrm{down}
\right\|
\le
\left\|J_\theta^\top\right\|
\left\|
\mathbf{w}^\top \tilde G - \tilde g_\mathrm{down}
\right\|,
\end{equation}
where \(\|\cdot\|\) denotes the spectral norm.
\end{theorem}

\begin{proof}
Using the chain rule,
\[
\mathbf{w}^\top G
=
J_\theta^\top
\mathbf{w}^\top \tilde G,
\qquad
g_\mathrm{down}
=
J_\theta^\top
\tilde g_\mathrm{down}.
\]
Therefore,
\begin{align}
\left\|
\mathbf{w}^\top G - g_\mathrm{down}
\right\|
&=
\left\|
J_\theta^\top
\left(
\mathbf{w}^\top \tilde G
-
\tilde g_\mathrm{down}
\right)
\right\| \\
&\le
\left\|J_\theta^\top\right\|
\left\|
\mathbf{w}^\top \tilde G
-
\tilde g_\mathrm{down}
\right\|,
\end{align}
where the last step follows from submultiplicativity of the operator norm.
\end{proof}

Thus, minimizing the embedding-space mismatch also minimizes the corresponding parameter-space mismatch up to a Jacobian-dependent constant factor. This motivates replacing parameter-space gradient comparisons with embedding-space comparisons.

\paragraph{Normalization-factor approximation.}
The final GraP objective additionally replaces the parameter-space normalization factor
\[
\left\|\mathbf{w}^\top G\right\|
\]
with its embedding-space counterpart
\[
\left\|\mathbf{w}^\top \tilde G\right\|.
\]

Let
\[
x = \mathbf{w}^\top \tilde G.
\]
Since
\[
\mathbf{w}^\top G = J_\theta^\top x,
\]
the denominator can be related to the embedding-space norm through the spectral properties of the Jacobian.

\begin{theorem}
Assume that the embedding-space alignment is nonnegative:
\[
x^\top \tilde g_{\mathrm{down}} \ge 0.
\]
Then
\begin{equation}
\frac{x^\top \tilde g_{\mathrm{down}}}
{\|\mathbf{w}^\top G\|}
\ge
\frac{1}{\sigma_{\max}(J_\theta^\top)}
\frac{x^\top \tilde g_{\mathrm{down}}}
{\|x\|},
\end{equation}
where \(\sigma_{\max}(J_\theta^\top)\) denotes the largest singular value of \(J_\theta^\top\).
\label{thm:lowerbound}
\end{theorem}

\begin{proof}
Since
\[
\mathbf{w}^\top G = J_\theta^\top x,
\]
the spectral norm inequality gives
\[
\|\mathbf{w}^\top G\|
=
\|J_\theta^\top x\|
\le
\sigma_{\max}(J_\theta^\top)\|x\|.
\]
Because the numerator is assumed to be nonnegative,
\[
x^\top \tilde g_{\mathrm{down}} \ge 0,
\]
dividing by the denominator preserves the inequality:
\[
\frac{x^\top \tilde g_{\mathrm{down}}}
{\|\mathbf{w}^\top G\|}
\ge
\frac{x^\top \tilde g_{\mathrm{down}}}
{\sigma_{\max}(J_\theta^\top)\|x\|}.
\]
Substituting \(x=\mathbf{w}^\top \tilde G\) completes the proof.
\end{proof}

Thus, the normalized embedding-space objective used in GraP can be interpreted as maximizing a lower-bound surrogate of the corresponding parameter-space objective up to a Jacobian-dependent constant. As shown in Appendix~\ref{app:trajectories}, the alignment remains positive throughout training in our experiments, making the assumption of Theorem~\ref{thm:lowerbound} applicable to the updates performed in practice.

\section{Training Details}
\label{app:training-details}
Here we summarize the main hyperparameters and training details. Additional information is available in the configuration files provided with the source code.

\paragraph{Event Sequences.}
An overview of the datasets is presented in Table~\ref{tab:seq-datasets}.
\begin{table*}[h]
\centering
\begin{tabular}{l|ccc|ccc}
\multirow{2}{*}{Dataset} & \multirow{2}{*}{\# Seq.} & \multirow{2}{*}{\thead{Fields / \\ \# Losses}} & Mean & \multicolumn{3}{c}{Downstream} \\
& & & length & Target & Classes & Metric \\
\hline
Churn & 10217 & 6 & 99.3 & Churn & 2 & ROC AUC \\
AgePred & 50000 & 3 & 875 & Age group & 4 & Accuracy \\
Alfabattle & 1466527 & 18 & 234 & Default & 2 & ROC AUC \\
MIMIC-III & 52103 & 3 & 407 & Mortality & 2 & ROC AUC \\
Taobao & 9904 & 3 & 527 & Activity & 2 & ROC AUC \\
MBD mini & 98721 & 13 & 372 & Credit & 2 & ROC AUC \\
\end{tabular}
\caption{Sequential datasets statistics}
\label{tab:seq-datasets}
\end{table*}

We follow the training setup from the HT-Transformer paper~\cite{karpukhin2025ht}. All models are trained using the Adam optimizer~\cite{kingma2014adam} with a fixed learning rate of \(0.001\). Depending on the dataset, training runs for up to 60--120 epochs with early stopping based on validation performance. Experiments are conducted on NVIDIA A100 GPUs: two GPUs for all datasets except AlfaBattle, which uses four GPUs. Dataset-specific hyperparameters are listed in Table~\ref{tab:dataset-hyperparameters}.
\begin{table}[h]
\centering
\begin{tabular}{l|cccccc}
    \multirow{2}{*}{Dataset} & \multirow{2}{*}{Epochs} & \multicolumn{2}{c}{Hidden layer} & Transf. & \multicolumn{2}{c}{Positional encoding} \\
    & & RNN & Transf. & layers & m & M \\
    \hline
    Churn & 120 & 256 & 256 & 4 & 4.2 & 6500 \\
    AgePred & 45 & 1536 & 512 & 8 & 0.003 & 8 \\ 
    Alfabattle & 60 & 1024 & 512 & 10 & 0.2 & 500\\
    MIMIC-III & 30 & 1280 & 512 & 8 & 1 & 4000 \\
    Taobao & 30 & 1280 & 512 & 8 & 0.008 & 30 \\
    MBD mini & 45 & 1024 & 512 & 8 & 0.004 & 8 \\
    \hline
\end{tabular}
\caption{Dataset-specific parameters.}
\label{tab:dataset-hyperparameters}
\end{table}

\paragraph{Computer vision.}
For the computer vision experiments, we use the solo-learn benchmark~\cite{da2022solo}. All models are trained with the LARS optimizer~\cite{you2017lars}. We train for 1000 epochs on CIFAR-10 and CIFAR-100, and for 400 epochs on ImageNet-100. Training is performed on a single A100 GPU for the CIFAR datasets and on two A100 GPUs for ImageNet-100.

\section{Example Weights Trajectories}
\label{app:trajectories}
Below we present the plots of the tuned weights dynamics during training:
\begin{itemize}
    \item Churn: Fig.~\ref{fig:traj-churn},
    \item AgePred: Fig.~\ref{fig:traj-agepred},
    \item AlfaBattle: Fig.~\ref{fig:traj-alfabattle},
    \item MIMIC-III: Fig.~\ref{fig:traj-mimiciii},
    \item Taobao: Fig.~\ref{fig:traj-taobao},
    \item CIFAR-10: Fig.~\ref{fig:traj-cifar10},
    \item CIFAR-100: Fig.~\ref{fig:traj-cifar100},
    \item ImageNet-100: Fig.~\ref{fig:traj-imagenet100}.
\end{itemize}

\begin{figure}[p]
    \centering
    \includegraphics[width=\textwidth]{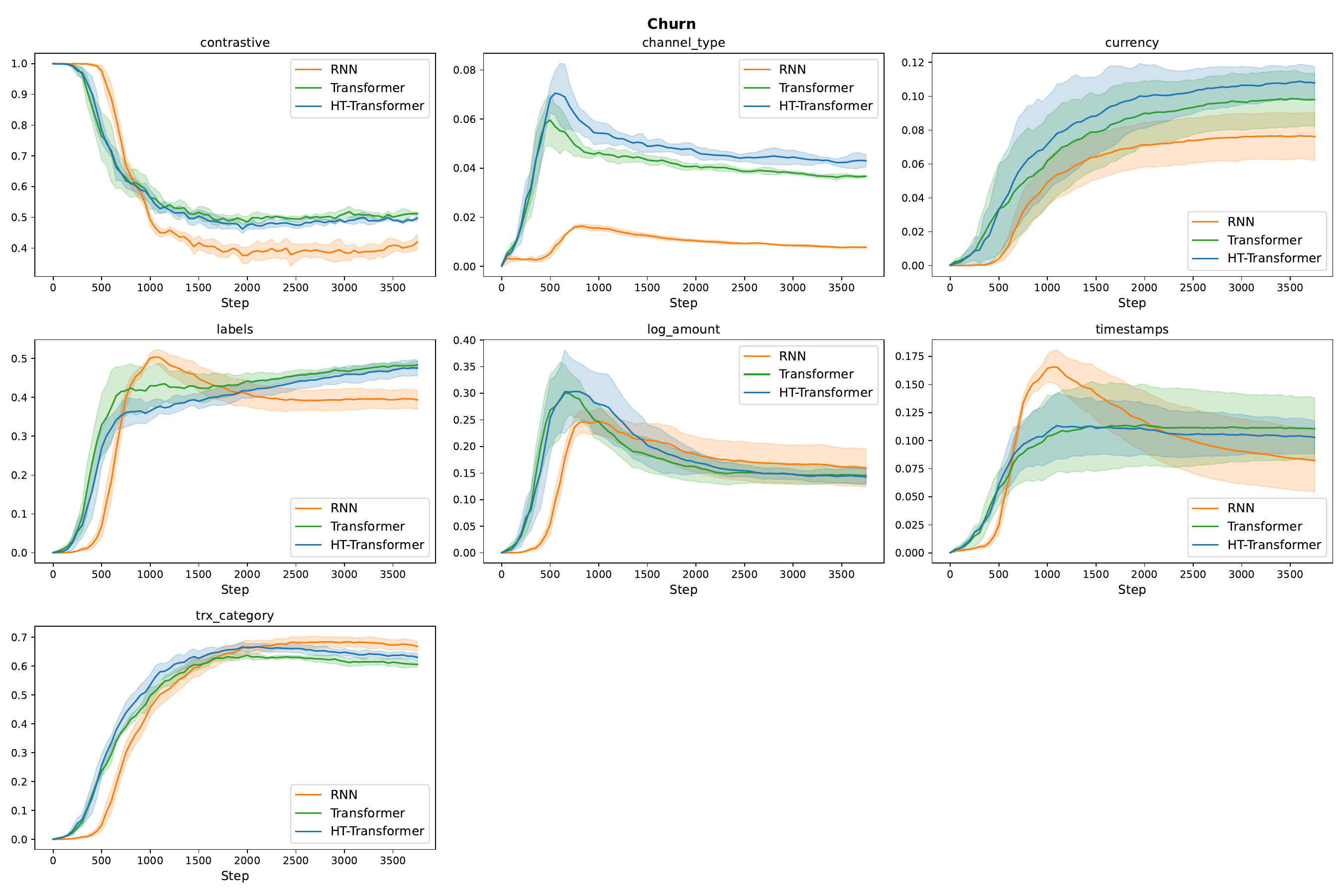}
    \caption{Weights trajectories for the Churn dataset.}
    \label{fig:traj-churn}
\end{figure}

\begin{figure}[p]
    \centering
    \includegraphics[width=\textwidth]{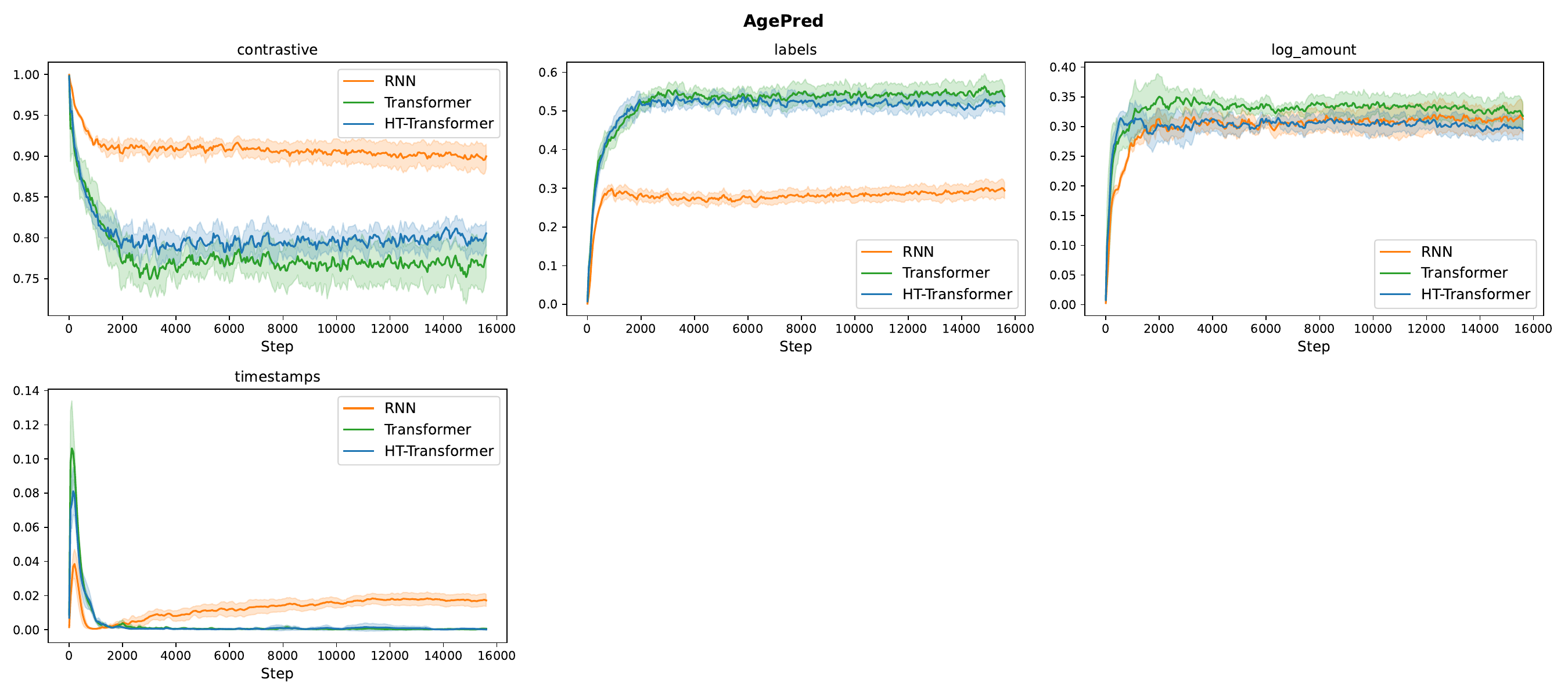}
    \caption{Weights trajectories for the AgePred dataset.}
    \label{fig:traj-agepred}
\end{figure}

\begin{figure}[p]
    \centering
    \includegraphics[width=\textwidth]{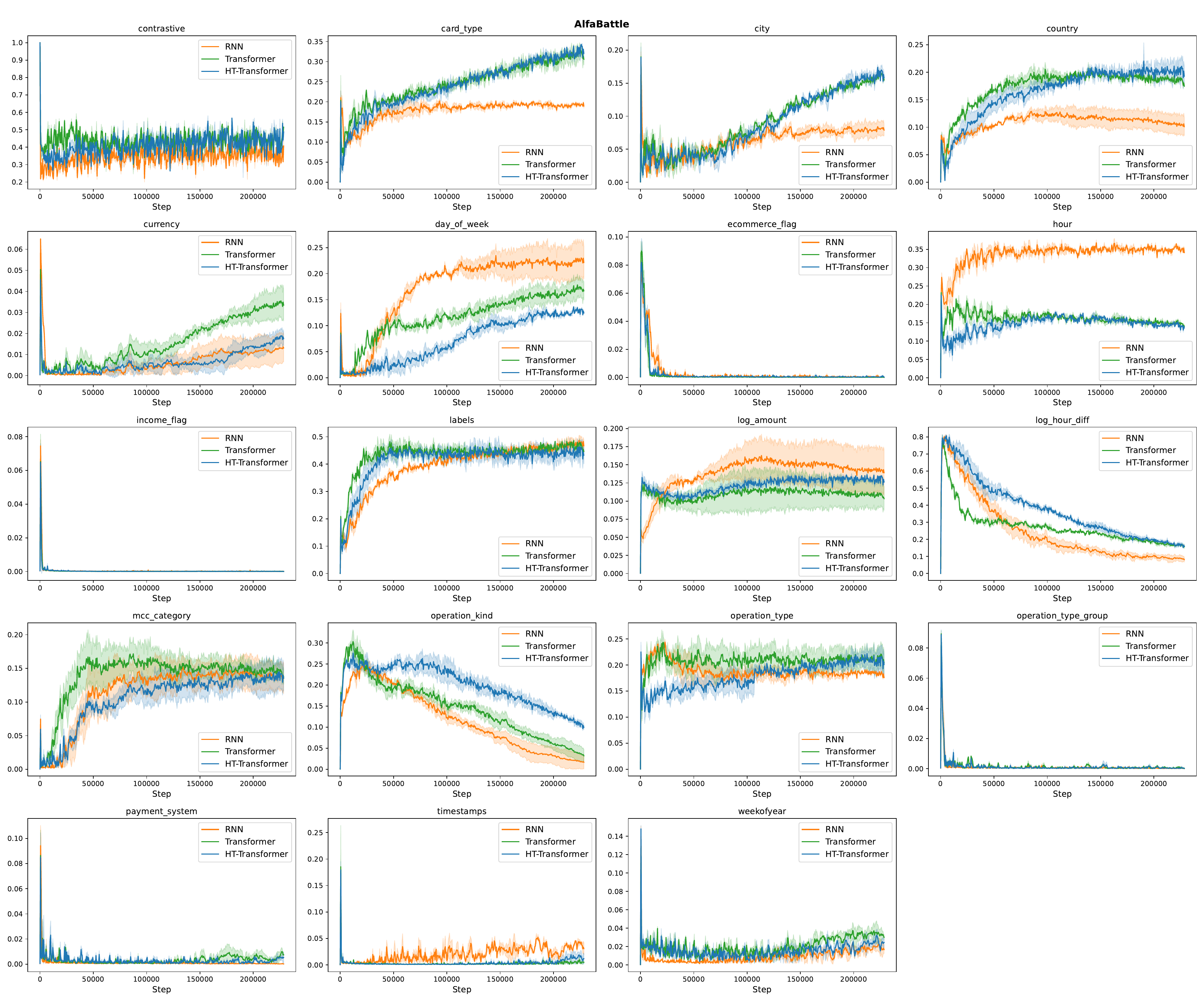}
    \caption{Weights trajectories for the AlfaBattle dataset.}
    \label{fig:traj-alfabattle}
\end{figure}

\begin{figure}[p]
    \centering
    \includegraphics[width=\textwidth]{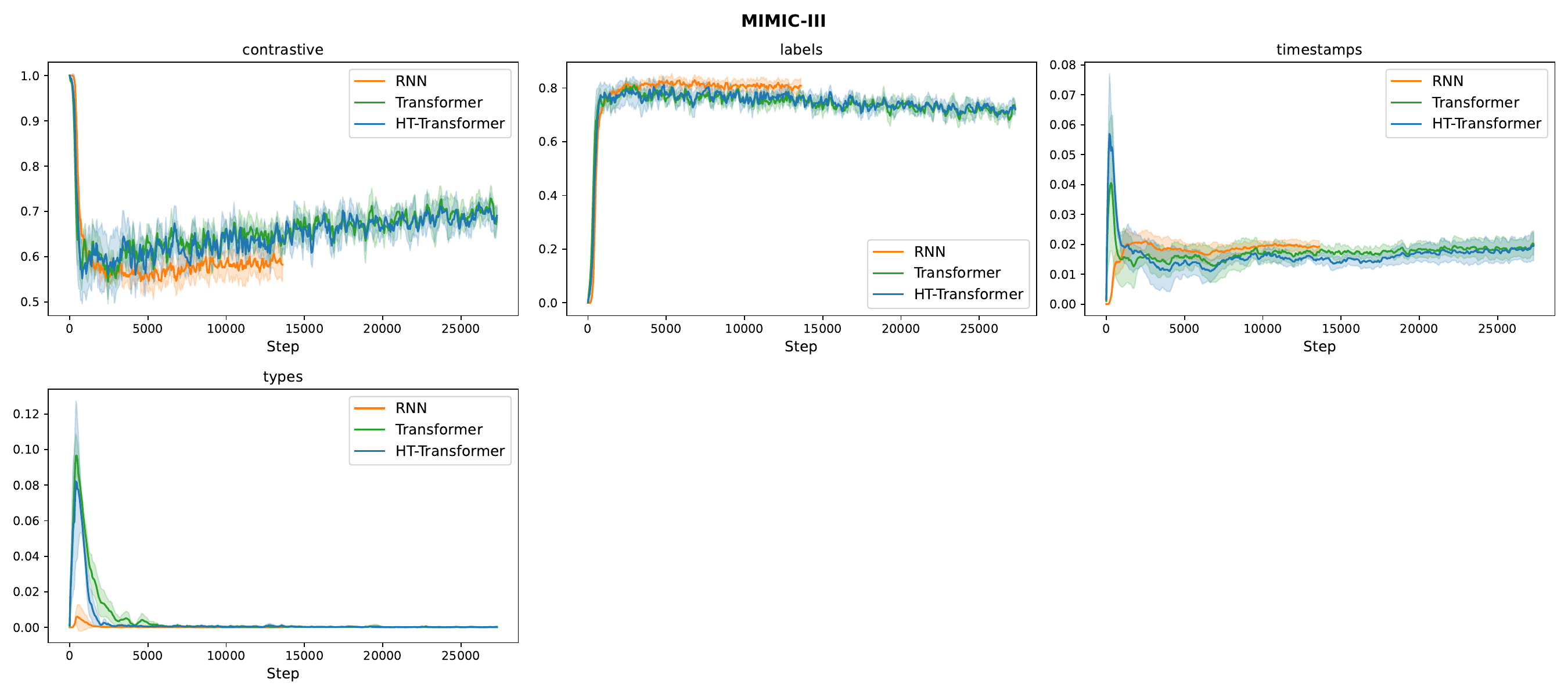}
    \caption{Weights trajectories for the MIMIC-III dataset.}
    \label{fig:traj-mimiciii}
\end{figure}

\begin{figure}[p]
    \centering
    \includegraphics[width=\textwidth]{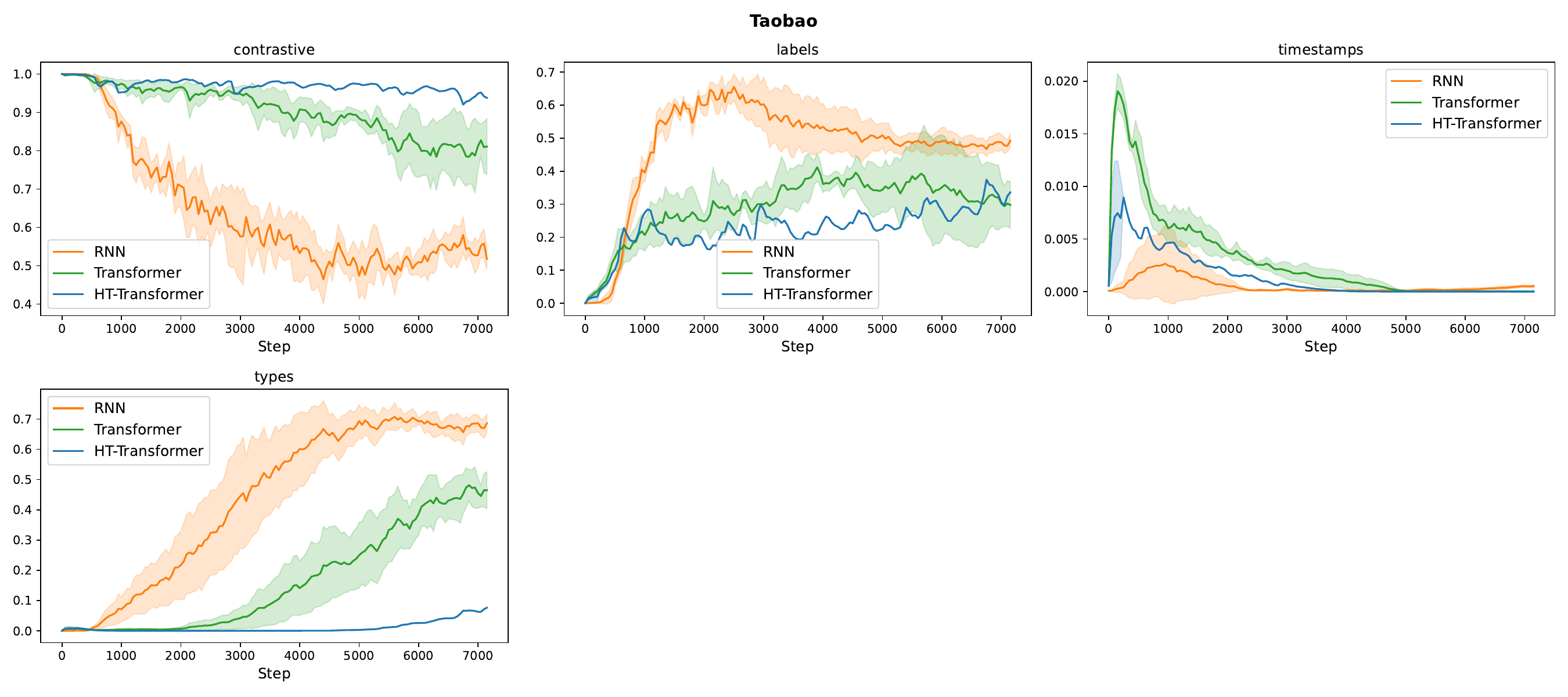}
    \caption{Weights trajectories for the Taobao dataset.}
    \label{fig:traj-taobao}
\end{figure}

\begin{figure}[p]
    \centering
    \includegraphics[width=\textwidth]{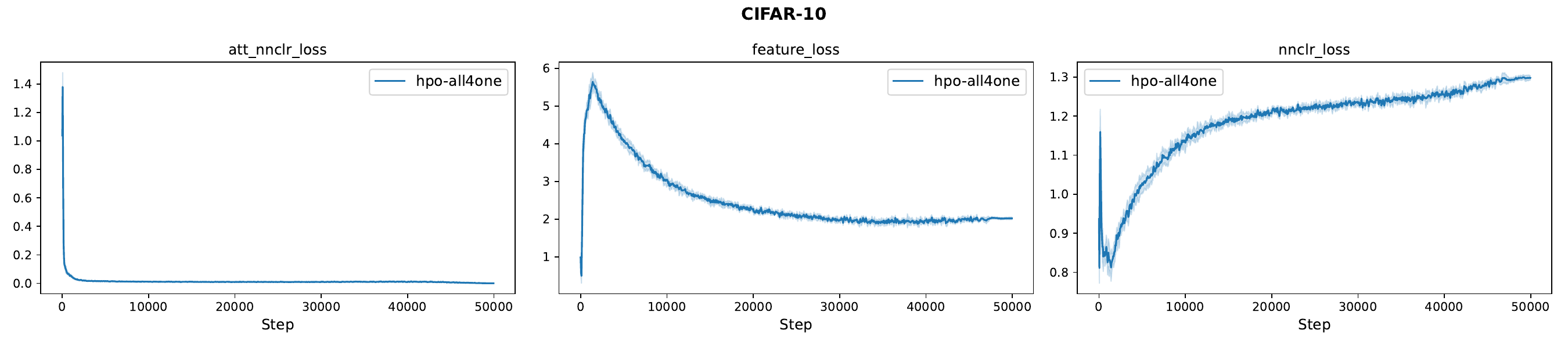}
    \caption{Weights trajectories for the CIFAR-10 dataset.}
    \label{fig:traj-cifar10}
\end{figure}

\begin{figure}[p]
    \centering
    \includegraphics[width=\textwidth]{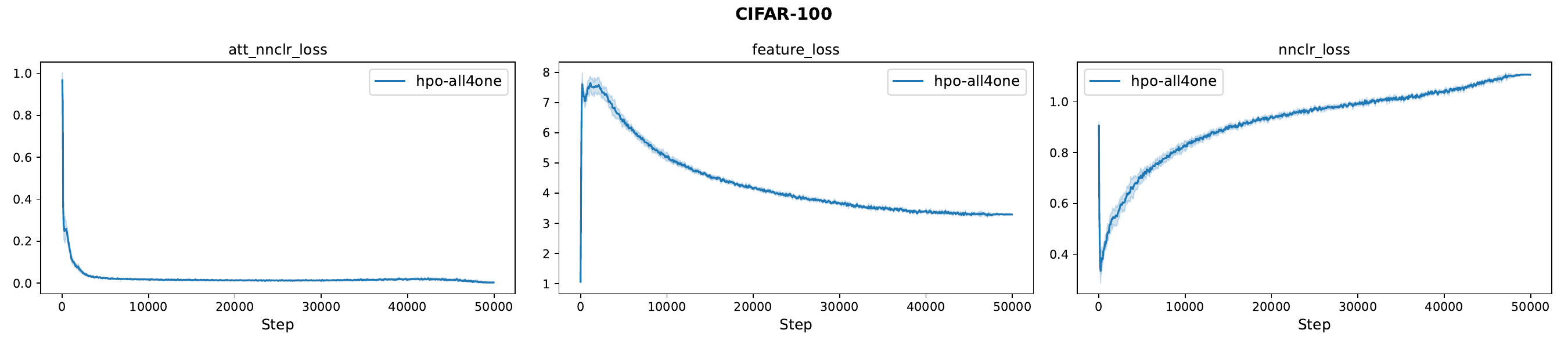}
    \caption{Weights trajectories for the CIFAR-100 dataset.}
    \label{fig:traj-cifar100}
\end{figure}

\begin{figure}[p]
    \centering
    \includegraphics[width=\textwidth]{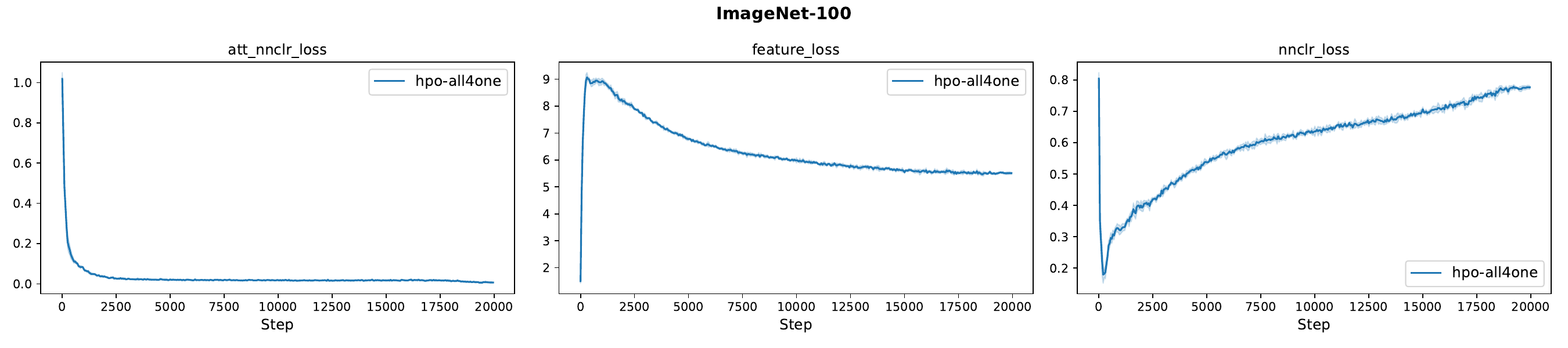}
    \caption{Weights trajectories for the ImageNet-100 dataset.}
    \label{fig:traj-imagenet100}
\end{figure}

\end{document}